\documentclass{article}

\usepackage[preprint]{corl_2026} %

\usepackage{amsmath, amssymb, amsthm}
\usepackage{graphicx}
\usepackage{algorithm}
\usepackage{algpseudocode}
\usepackage{tikz}
\usepackage{enumitem}
\usepackage{subcaption}
\usepackage{booktabs}
\usepackage{siunitx}
\usepackage{wrapfig}
\usepackage{float}
\usetikzlibrary{decorations.pathmorphing,arrows.meta,calc,angles,quotes}
\theoremstyle{definition}
\newtheorem{definition}{Definition}
\theoremstyle{plain}
\newtheorem{proposition}{Proposition}

\title{Actuator Dynamics Curricula for Narrow-Viability Tasks in Legged Robot Learning}

\author{
  Kousheek Chakraborty$^{1,2}$, \;
  Chandan K. Rajendra$^{1}$, \; 
  Ayham Alharbat$^{1,3}$, \;
  Abeje Y. Mersha$^{1}$\\
  $^{1}$ Saxion University of Applied Sciences, $^{2}$ University of Groningen, $^{3}$ University of Twente
  \\
  {\texttt{k.chakraborty@\{saxion,rug\}.nl}}
}

\begin{document}
\maketitle

\begin{abstract}
Reinforcement learning has produced capable controllers across a broad range of legged-robot tasks, but a subset of these tasks fail to converge under standard training: those for which most exploration trajectories terminate before producing useful gradient signal. To address such tasks we introduce the \emph{Actuator Dynamics Curriculum}, a procedure that initializes joint stiffness at a high value and anneals it toward the system-identified value as completed episode lengths grow. Using a cart-pole system as a representative example, we show that higher closed-loop joint natural frequency under critical damping enlarges the viability kernel of the underlying Markov Decision Process, increasing the fraction of initial states from which the task is feasible. We validate the kernel monotonicity on the cart-pole and apply the curriculum to a quadrupedal-to-handstand transition on the Boston Dynamics Spot, a narrow-viability task where training under fixed identified stiffness plateaus at a policy that never completes the transition. The trained policy executes the transition in simulation across 10 seeds and transfers to hardware. More broadly, our results suggest that simulated actuator dynamics is a useful axis along which to design curricula for tasks in which exploration is bottlenecked by termination conditions rather than by reward signal.
\end{abstract}

\keywords{Legged Robots, Reinforcement Learning, Curriculum Learning}

\begin{figure}[h]
  \centering
  \begin{subfigure}[b]{0.49\linewidth}
    \includegraphics[width=\linewidth, trim={0 0 0.5cm 0}, clip]{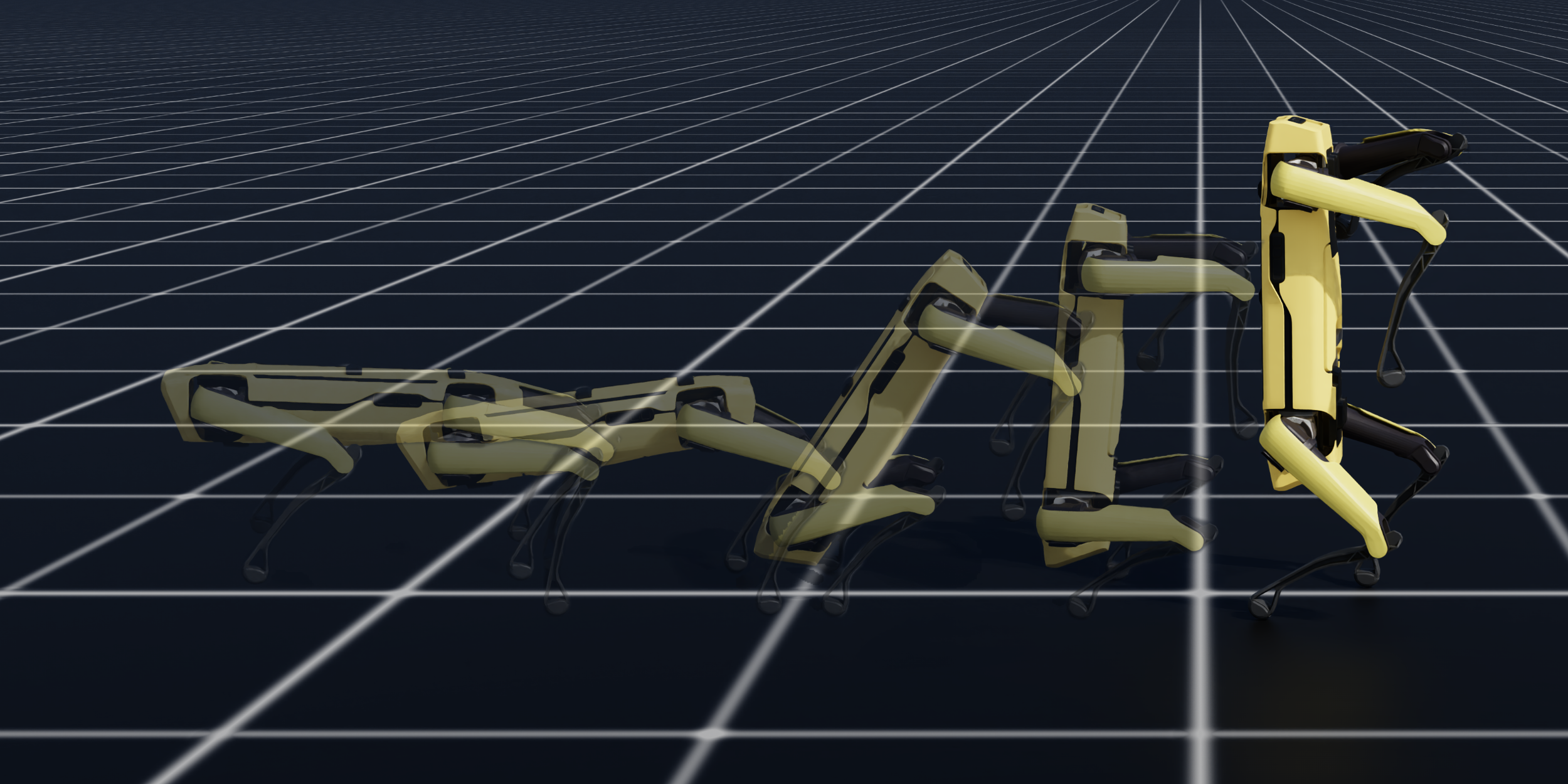}
    \caption{Simulation}
    \label{fig:cover-sim}
  \end{subfigure}
  \hfill
  \begin{subfigure}[b]{0.49\linewidth}
    \includegraphics[width=\linewidth, trim={0.5cm 0 0 0}, clip]{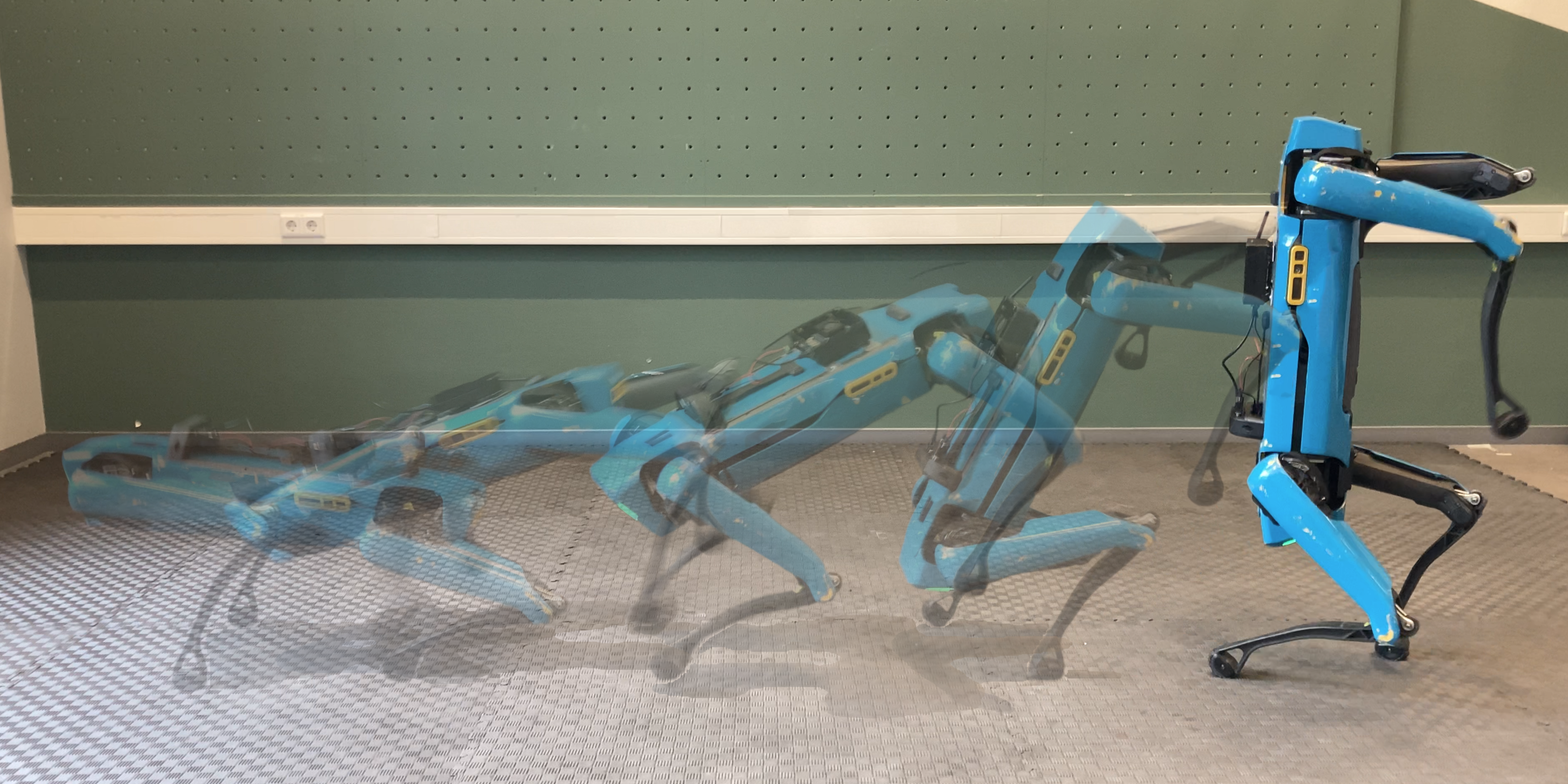}
    \caption{Hardware}
    \label{fig:cover-real}
  \end{subfigure}
  \caption{
    Boston Dynamics' Spot quadruped executing a transition to a handstand from a four-legged standing stance. The policy is trained and tested first in simulation~(\subref{fig:cover-sim}) and transfers zero-shot to a physical Spot~(\subref{fig:cover-real}).
  }
  \label{fig:spot-rl}
\end{figure}

\section{Introduction}
\label{sec:introduction}

Reinforcement learning (RL) has enabled increasingly capable motion planners and controllers for legged robots, including locomotion across diverse terrains~\cite{hwangboLearningAgileDynamic2019, leeLearningQuadrupedalLocomotion2020, rudinParkourWildLearning2025, kumarRMARapidMotor2021, zuckerOptimizationApproachRough2010, Zucker-2011-121524}, dynamic jumping~\cite{chengExtremeParkourLegged2024a}, and bipedal balance on quadrupedal platforms~\cite{xiaoLearningStableBipedal2025, zhangBipedalismQuadrupedalRobots2025}.
A common ingredient across these results is large-scale parallel simulation~\cite{rudinLearningWalkMinutes2022a}, which compensates for the sample inefficiency of policy gradient methods by collecting orders of magnitude more environment interactions per wall-clock second.
However, for a subset of robot learning tasks, parallelization alone is insufficient: the set of trajectories that both pursue reward and avoid early termination is vanishingly small under any practical sampling budget~\cite{heimLearningOutsideViability2018}.
We refer to these as \emph{narrow-viability} tasks.

A representative instance is the transition from a quadrupedal stance to a balanced handstand on the front legs of a Boston Dynamics Spot.
The maneuver requires the body to tip forward by $90^\circ$ within roughly one second, during which the hind legs lose ground contact, and the front legs take the load and stabilize the robot's body.
Throughout this trajectory, the robot passes through a region of state space: mid-pitch, zero-contact, rotating, in which most actions cause the robot to fall and trigger an early termination of the training.

One approach to address this class of tasks is imitation learning, which sidesteps the exploration problem by training from expert demonstrations~\cite{vollenweiderAdvancedSkillsMultiple2023, fuchiokaOPTMimicImitationOptimized2023a, pengLearningBipedalWalking2025, xiaoStableImitationMultigait2025}.
However, such demonstrations are unavailable for many platforms and behaviors of interest.
A more general alternative is curriculum learning, which varies task difficulty during training~\cite{narvekarCurriculumLearningReinforcement2020}.
In our setting, after system identification on the Spot~\cite{millerHighPerformanceReinforcementLearning2025}, the natural choice was to train under the identified PD parameters, but we observed that training failed to converge under those gains, which prompted us to consider whether the simulated actuator dynamics themselves could be varied as part of the curriculum.
Specifically what we propose is that each joint is initialized with stiffness higher than its system-identified value and corresponding damping chosen to keep the closed-loop joint dynamics critically damped, and the stiffness is then annealed to the identified value as the policy improves.

The intuition is that higher joint stiffness results in more aggressive tracking of commanded positions, giving the policy more margin to recover from imperfect actions during early training; as the policy improves, the dynamics are returned to those of the hardware so that the deployed policy operates under the actuators it will actually see.
We refer to this procedure as an \emph{Actuator Dynamics Curriculum}.
It is parametrized by a single physically meaningful scalar, the joint stiffness, and requires only that actuator dynamics can be varied at simulation time, a capability that is available in modern physics engines used for RL training~\cite{NVIDIA_Isaac_Sim}.

We support this approach by analyzing a representative cart-pole system in which the cart position is tracked by a PD controller. We show that increasing the closed-loop joint natural frequency under critical damping enlarges the viability kernel~\cite{aubinViabilityTheory2009} of the task or the set of states from which task constraints can be satisfied indefinitely~\cite{heimLearningOutsideViability2018} and hence increases the fraction of initial states from which the task is feasible. This result provides a motivation for training under elevated stiffness because in the representative bounded-command system, higher natural frequency enlarges the set of states from which task constraints can be maintained. However, in the full RL setting, this result does not by itself establish viability-kernel enlargement as the mechanism responsible for improved policy learning. We investigate this empirically in Section~4. The statement and proof appear in Section~\ref{sec:problem} and Appendix~\ref{apx-proofs}.

\subsection{Contributions}
This work makes the following contributions:
\begin{enumerate}
    \item We introduce the \emph{Actuator Dynamics Curriculum}, a training procedure that anneals simulated joint stiffness from a high initial value to its system-identified hardware value while keeping the closed-loop joint dynamics critically damped (Section~\ref{sec:method}).
    \item We characterize narrow-viability as a class of robot learning tasks in which most trajectories under random exploration terminate early, leaving policy-gradient methods without sufficient learning signal to converge (Section~\ref{sec:problem}).
    \item We give an analytical result on a representative cart-pole system: increasing the natural frequency of a position-tracking PD controller under critical damping enlarges the viability kernel and the fraction of initial states from which the task is feasible (Section~\ref{sec:problem}, with proof in Appendix~\ref{apx-proofs}).
    \item We evaluate the method on the handstand task in simulation and demonstrate sim-to-real transfer of the resulting handstand transition policy on a physical Spot robot (Section~\ref{sec:experiments}).
\end{enumerate}

\subsection{Related Work}
\label{sec:related_work}
Acrobatic and bipedal behaviors on quadrupedal robots have been studied through a range of methods. Model-based trajectory optimization has been used to produce a 360\textdegree{} backflip on Mini Cheetah, by solving an offline optimization that respects the robot's actuator limits~\cite{katzMiniCheetahPlatform2019}. Other approaches use reference motions to guide RL training. Adversarial motion priors~\cite{vollenweiderAdvancedSkillsMultiple2023} learn a set of switchable styles on a wheeled-legged quadruped, including transitions between quadrupedal and humanoid configurations, by imitating motion-capture data. The same family of methods has been applied to blind bipedal walking on a standard quadruped through a teacher-student framework~\cite{pengLearningBipedalWalking2025}, and to multigait and bipedal motions over uneven terrain through imitation~\cite{xiaoStableImitationMultigait2025}. Style-reward formulations have been used to enable a quadruped to switch between quadrupedal and bipedal locomotion modes within a single policy~\cite{queLearningHumanLikeBipedal2025}. OPT-Mimic~\cite{fuchiokaOPTMimicImitationOptimized2023a} replaces motion-capture references with trajectory-optimization rollouts and produces front hops, a 180\textdegree{} backflip, and biped stepping on a Solo 8 quadruped.

Other agile behaviors have been learned without reference motions. An extreme-parkour controller~\cite{chengExtremeParkourLegged2024a} learns end-to-end from depth images on a curriculum of increasingly difficult terrains, producing high jumps, long jumps, and a handstand among its skills. TumblerNet~\cite{xiaoLearningStableBipedal2025} uses a center-of-mass/center-of-pressure estimator with a balance-targeting reward to learn bipedal locomotion on quadrupeds, and risk-adaptive distributional RL has been applied to bipedal walking on a Unitree Go2~\cite{zhangBipedalismQuadrupedalRobots2025}.

The choice of action representation and low-level control has also been studied in robot learning. Different action-space parameterizations have been compared for learning robot motor skills~\cite{esser2025action}, and their effects on manipulation learning and sim-to-real transfer have been analyzed~\cite{aljalbout2024role}. Low-level controller gains have likewise been shown to substantially affect policy learning~\cite{bronars2026tune}. Our approach instead explicitly schedules the closed-loop joint dynamics during training and anneals them toward the system-identified deployment values.

The approach we present in this paper differs from the above in what is changed during training. Rather than giving the policy external information such as a precomputed reference trajectory, a motion-prior dataset, or a hand-tuned reward, we change the simulated robot's actuator dynamics during training and anneal them back toward their system-identified values.

\begin{figure}[t]
  \centering
  \includegraphics[width=\linewidth]{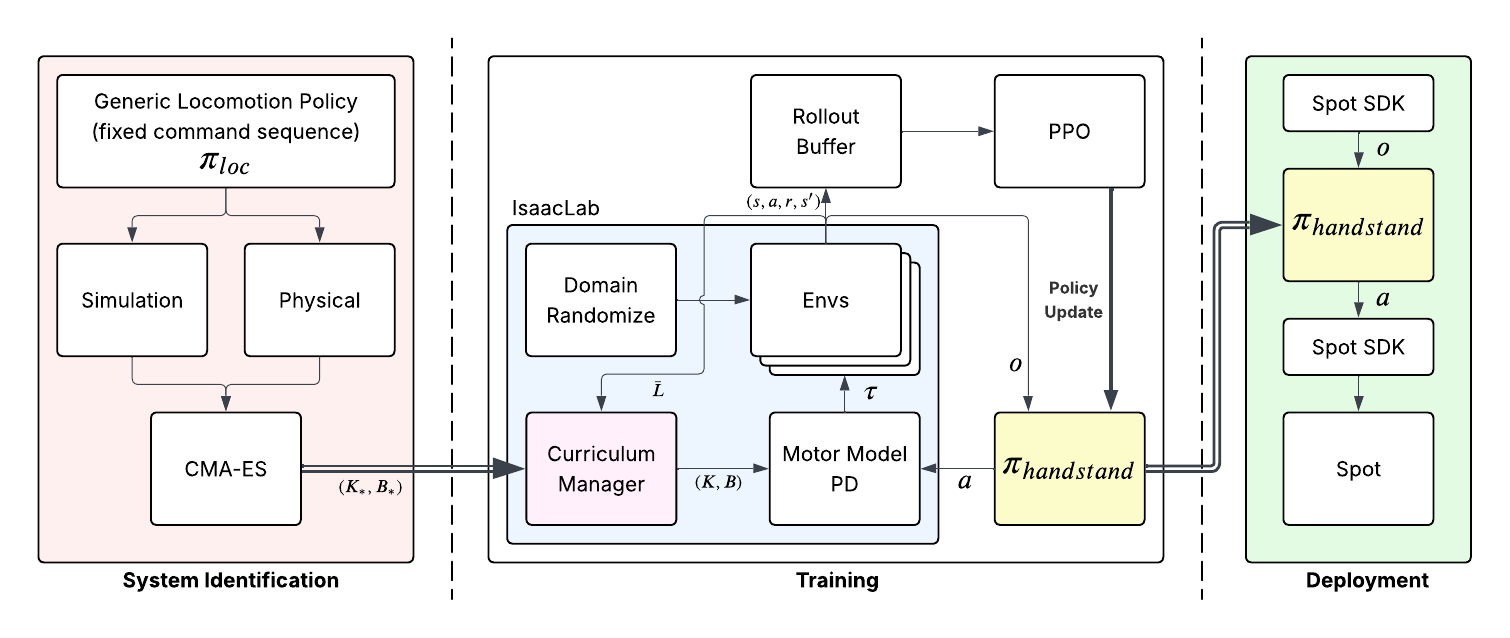}
  \caption{\textbf{System overview.} System identification of actuator parameters $K_*$ is performed via CMA-ES against trajectories collected on a Boston Dynamics Spot running a generic locomotion policy. Training proceeds in IsaacLab with PPO under the Actuator Dynamics Curriculum and domain randomization. The trained policy is deployed to an NVIDIA Jetson Nano onboard the robot, which communicates with the hardware via the Spot SDK.}
  \label{fig:architecture}
\end{figure}

\section{Problem Formulation}
\label{sec:problem}

\subsection{Setup}
\label{sec:setup}

We model the robot and task as a Markov decision process $\mathcal{M}_\kappa = (\mathcal{S}, \mathcal{A}, P_\kappa, r, \gamma, \rho_0, \mathcal{S}_\text{term})$ parameterized by the actuator parameters $\kappa$. States $s \in \mathcal{S}$ and actions $a \in \mathcal{A}$ are continuous, $\gamma \in [0,1)$ is the discount factor, and $\rho_0$ is the initial state distribution. The transition kernel $P_\kappa(s' \mid s, a)$ is induced by the rigid-body dynamics of the robot and the actuator model parameterized by $\kappa$ such that varying $\kappa$ changes the closed-loop dynamics of the actuators and consequently the distribution of states visited under a policy.
The reward $r : \mathcal{S} \times \mathcal{A} \to \mathbb{R}$ is dense and shaped to reward task progress.
The termination set $\mathcal{S}_\text{term} \subset \mathcal{S}$ encodes safety violations such as illegal contact or body collision and entering $\mathcal{S}_\text{term}$ ends the episode and truncates the return.
We optimize a parameterized policy $\pi_\theta(a \mid s)$ to maximize the expected discounted return,
\begin{equation}
    J(\theta;\kappa)
    =
    \mathbb{E}_{s_0 \sim \rho_0,\, \tau \sim \pi_\theta,\, P_\kappa}
    \left[
        \sum_{t=0}^{\infty} \gamma^t r(s_t,a_t)
    \right],
    \label{eq:objective}
\end{equation}
where trajectories are generated under dynamics $P_\kappa$.

The policy outputs an action $a \in \mathcal{A} \subset  \mathbb{R}^n$ for a robot with $n$ joints. The commanded joint position setpoints are defined as offsets from a nominal configuration $q_0 \in \mathbb{R}^n$: $q_d = q_0 + \alpha\, a$, where $\alpha > 0$ is a fixed scaling factor. In the experiments, actions are clipped component-wise to $a_i \in [-1,1]$, resulting in bounded relative joint-position commands.
For each joint $i$, a low-level PD controller produces the actuator torque,
\begin{equation}
    \tau_i = -K_i (q_i - q_{d, i}) - B_i \dot{q}_i,
    \label{eq:pd_law}
\end{equation}
where $K_i > 0$ is the proportional gain (joint stiffness) and $B_i > 0$ is the derivative gain (joint damping).
Under a per-joint idealization in which $M_i$ denotes the effective inertia at joint $i$ and cross-joint coupling is absorbed into $\tau_{\text{ext},i}$, the closed-loop dynamics of joint $i$ can be expressed as,
\begin{equation}
    M_i \ddot{q}_i + B_i \dot{q}_i + K_i (q_i - q_{d,i}) = \tau_{\text{ext}, i},
    \label{eq:actuator}
\end{equation}
where $\tau_{\text{ext}, i}$ collects gravity, contacts, Coriolis terms, and inter-joint coupling. We parameterize the closed-loop joint dynamics by the natural frequency $\omega_{n,i} = \sqrt{K_i / M_i}$ and the damping ratio $\zeta_i = B_i / (2 \sqrt{K_i M_i})$.
Through $P_\kappa$, the closed-loop joint parameters enter both the viability kernel and the policy-gradient landscape of $\mathcal{M}_\kappa$; the analysis in Section~\ref{sec:propositions} exploits this dependence.
The Actuator Dynamics Curriculum varies $K_i$ during training while keeping $\zeta_i = 1$ to remain critically damped, so that only the closed-loop joint bandwidth changes while the damping properties are preserved. The curriculum schedule is specified in Section~\ref{sec:method}.

\subsection{Propositions}
\label{sec:propositions}
We use the standard notion of the viability kernel introduced in~\cite{aubinViabilityTheory2009}, later adapted to RL in~\cite{heimLearningOutsideViability2018, shafieeViabilityLeadsEmergence2024, massianiViabilityFutureActions2024}, restated here in finite-horizon form.

\begin{definition}[Finite-horizon viability kernel]
\label{def:viab}
For horizon $T \in \mathbb{N}$, the finite-horizon viability kernel of $\mathcal{M}_\kappa$ is the set of states from which there exists a policy that keeps the trajectory out of the termination set for $T$ steps with probability one,
\begin{equation}
    \mathcal{V}_T(\mathcal{M}_\kappa) = \left\{ s \in \mathcal{S} : \exists\, \pi,\; \mathbb{P}_{\pi, P_\kappa}\!\left[ s_t \notin \mathcal{S}_\text{term}\ \forall t \in \{0, \dots, T\} \mid s_0 = s \right] = 1 \right\}.
    \label{eq:viab}
\end{equation}
\end{definition}

States outside $\mathcal{V}_T(\mathcal{M}_\kappa)$ cannot be guaranteed to remain outside $\mathcal{S}_\text{term}$ for the full horizon $T$ under any policy. Building on this, we formalize the notion of a narrow-viability task as follows.

\begin{definition}[$\varepsilon$-narrow-viable]
\label{def:narrow_viability}
For $\varepsilon \in (0, 1)$ and horizon $T \in \mathbb{N}$, a task is \emph{$\varepsilon$-narrow-viable} with respect to an initial state distribution $\rho_0$ if a state sampled from the state visitation distribution induced by uniform random exploration lies in the finite-horizon viability kernel with probability at most $\varepsilon$; formally,
\begin{equation}
    \mathbb{P}_{s \sim d^{\pi_u, \rho_0}_T}\!\left[ s \in \mathcal{V}_T(\mathcal{M}_\kappa) \right] \;\leq\; \varepsilon,
    \label{eq:narrow_viability}
\end{equation}
where $\pi_u$ is the uniform random policy over the bounded action set $\mathcal{A}$, and $d^{\pi_u, \rho_0}_T$ is the state visitation distribution within $T$ steps under $\pi_u$ starting from $\rho_0$.
\end{definition}

The proposition below is stated for a representative cart-pole system, linearized about the upright equilibrium, with $\mathcal{M}_{\omega_n}$ denoting the resulting MDP at natural frequency $\omega_n$. The result specifically considers a bounded position-command interface, matching the relative position-command action spaces used in our experiments. The full setup, assumptions, and proof are in Appendix~\ref{apx-proofs}.

\begin{proposition}[Kernel monotonicity]
\label{prop:kernel}
For the representative cart-pole system considered in Appendix~\ref{apx-proofs}, under the stated bounded-command feasibility assumptions, the finite-horizon viability kernel of $\mathcal{M}_{\omega_n}$ is monotone non-decreasing with respect to the natural frequency:
\begin{equation}
    \omega_{n,1} \;\leq\; \omega_{n,2} \implies
    \mathcal{V}_T(\mathcal{M}_{\omega_{n,1}})
    \;\subseteq\;
    \mathcal{V}_T(\mathcal{M}_{\omega_{n,2}}).
    \label{eq:kernel_mono}
\end{equation}
\end{proposition}

Intuitively, a higher natural frequency gives the controller more authority and responsiveness to drive the state back toward a reference in a given time, enlarging the set of states from which termination can be avoided over the horizon.

\paragraph{Viable initial-state coverage.}
Define
\[
\mathcal{C}_T(\omega_n) \;\triangleq\; \mathbb{P}_{s_0\sim\rho_0}\!\left[s_0 \in \mathcal{V}_T(\mathcal{M}_{\omega_n})\right],
\]
the probability that the initial state lies in the kernel. By Proposition~\ref{prop:kernel}, $\mathcal{C}_T$ is non-decreasing in $\omega_n$:
\[
\omega_{n,1} \le \omega_{n,2} \;\implies\; \mathcal{C}_T(\omega_{n,1}) \le \mathcal{C}_T(\omega_{n,2}).
\]
Raising the natural frequency moves initial-state mass into the kernel and away from the narrow-viability regime of Definition~\ref{def:narrow_viability}.

\section{Method}
\label{sec:method}

Proposition~\ref{prop:kernel} and the resulting monotonicity of $\mathcal{C}_T$ motivate training at high $\omega_n$, but the policy must ultimately be deployed under the identified hardware value $K_*$. Motivated by this, the Actuator Dynamics Curriculum holds training in the high-stiffness regime $K_0$, until episode length grows, then anneals toward $K_*$ as training progresses. Damping is recomputed each iteration to maintain critical damping ($\zeta = 1$), so that only the closed-loop joint bandwidth changes; the annealing is driven by an exponential moving average of completed episode lengths.

The stiffness $K(t)$ at training iteration $t$ is set by
\begin{equation}
    K(t) = K_0 + \beta(t)^p \, (K_* - K_0),
    \label{eq:curriculum_schedule}
\end{equation}
where $\beta(t) \in [0, 1]$ is the curriculum progress and $p > 0$ controls the shape of the annealing curve (we use $p = 1$, giving linear interpolation). The stiffness is applied uniformly to all joints; the damping at joint $i$ is set to $B_i(t) = 2 \sqrt{K(t)\, M_i}$ to maintain critical damping per joint.

The curriculum progress $\beta(t)$ is computed from a running estimate $\bar{L}(t)$ of the average length of completed episodes,
\begin{equation}
    \beta(t) = \mathrm{clip}\!\left(\frac{\bar{L}(t) - L_{\min}}{L_{\max} - L_{\min}},\, 0,\, 1\right),
    \label{eq:curriculum_progress}
\end{equation}
where $L_{\min}$ and $L_{\max}$ specify the episode-length range over which the schedule moves between $K_0$ and $K_*$. The estimate $\bar{L}(t)$ is maintained as an exponential moving average over completed episodes with smoothing factor $\eta \in (0, 1)$. The full training procedure is summarized in Algorithm~\ref{alg:curriculum} and the curriculum hyperparameters used in the  experiments are reported in Table~\ref{tab:adc_hyperparams}.

\begin{algorithm}[b]
\caption{Actuator Dynamics curriculum}
\label{alg:curriculum}
\begin{algorithmic}[1]
\State \textbf{Input:} initial stiffness $K_0$, target stiffness $K_*$, episode-length thresholds $L_{\min}, L_{\max}$, EMA factor $\alpha$, exponent $p$, joint inertias $\{M_i\}$
\State Initialize $\bar{L} \leftarrow 0$, $K \leftarrow K_0$, $B_i \leftarrow 2\sqrt{K M_i}$ for each joint $i$
\For{each training iteration $t = 1, 2, \ldots$}
    \State Collect rollouts under policy $\pi_\theta$ with actuator parameters $(K, \{B_i\})$
    \State Update $\pi_\theta$ via policy gradient
    \State $\bar{L}_\text{batch} \leftarrow$ mean length of completed episodes in this iteration
    \State $\bar{L} \leftarrow (1 - \eta)\, \bar{L} + \eta\, \bar{L}_\text{batch}$
    \State $\beta \leftarrow \mathrm{clip}\!\left((\bar{L} - L_{\min}) / (L_{\max} - L_{\min}),\, 0,\, 1\right)$
    \State $K \leftarrow K_0 + \beta^p (K_* - K_0)$
    \State $B_i \leftarrow 2\sqrt{K M_i}$ for each joint $i$
\EndFor
\end{algorithmic}
\end{algorithm}

The curriculum holds training in the large-kernel regime until episode lengths grow, then transitions the dynamics to $K_*$, so that the deployed policy is trained on the actuator characteristics it will encounter on hardware.

\begin{figure}[t]
  \centering
  \begin{subfigure}[t]{0.48\linewidth}
    \centering
    \includegraphics[width=\linewidth]{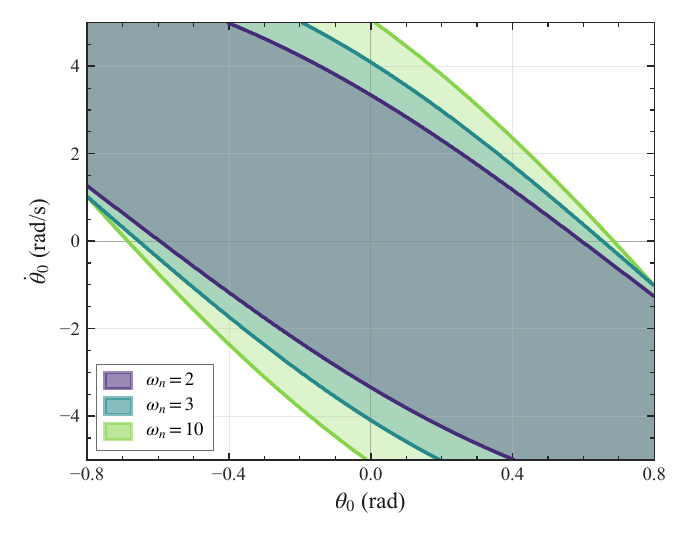}
    \caption{Slice of $\mathcal{V}_T(\mathcal{M}_{\omega_n})$ at $x_0 = \dot{x}_0 = 0$ for three values of $\omega_n$.}
    \label{fig:viab-kernel}
  \end{subfigure}\hfill
  \begin{subfigure}[t]{0.48\linewidth}
    \centering
    \includegraphics[width=\linewidth]{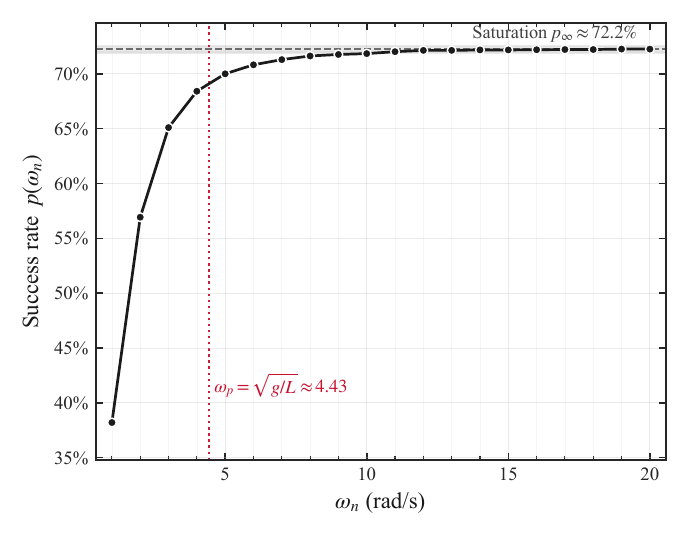}
    \caption{Success rate $p(\omega_n)$.}
    \label{fig:viab-rate}
  \end{subfigure}
  \caption{%
    The admissible region is $\{|x| \le 1\,\text{m},\ |\theta| \le \pi/2\}$. \textbf{(a)} The viable region grows monotonically with $\omega_n$. \textbf{(b)} The success rate increases monotonically with $\omega_n$.%
  }
  \label{fig:viability}
\end{figure}

\section{Experimental Results}
\label{sec:experiments}

\subsection{Experimental Setup}
\label{sec:experimental_setup}

Figure~\ref{fig:architecture} summarizes the training and deployment pipeline. Policies are trained in simulation using PPO~\cite{schulmanProximalPolicyOptimization2017} within the IsaacLab framework~\cite{nvidiaIsaacLabGPUAccelerated2025} with the RSL-RL library ~\cite{schwarkeRSLRLLearningLibrary2025} on a workstation equipped with an Intel Core i9 CPU, NVIDIA RTX 5000 Ada GPU, and 32GB of RAM. The full hyperparameter list can be found in Appendix~\ref{apx-train-sim-params}. End-to-end training of the policy takes approximately 1.5 hours of wall-clock time per seed. The system-identified joint parameters $K_*$ are obtained using a CMA-ES approach as described in~\cite{millerHighPerformanceReinforcementLearning2025}. The system-identified stiffness $K^*=40$ represents the effective closed-loop stiffness of the assembled joint under the Spot control interface. We select $K_0$ by increasing stiffness from $K^*$ in increments of 5 and choosing the smallest value that exceeds a pre-specified $98\%$ task-success threshold over 1000 simulation episodes, giving us $K_0=60$. At deployment, the trained policy is copied to an NVIDIA Jetson Nano mounted on the Boston Dynamics Spot, which runs the inference pipeline and communicates with the robot via the Spot SDK.

\subsection{Empirical Validation of Proposition 1}
\label{sec:cartpole_results}

We empirically validate Proposition~\ref{prop:kernel} on the cart-pole system from Section~\ref{sec:propositions}. For a range of natural frequencies $\omega_n$ under critical damping, we initialize the cart at the origin and sweep the pole angle $\theta_0 \in [-0.8, 0.8]$~rad and angular velocity $\dot{\theta}_0 \in [-5, 5]$~rad/s. For each initial condition we run a closed-loop simulation under an outer LQR designed on the linearized dynamics with the inner PD loop included in the plant model, and record whether the trajectory remains in the admissible region. Figure~\ref{fig:viability} shows the resulting viable region and success rate $p(\omega_n)$ over the sampled grid. The viable region grows monotonically with $\omega_n$, consistent with Proposition~\ref{prop:kernel}.

\subsection{Spot Handstand}
\label{sec:spot_handstand}

The handstand transition task requires the Boston Dynamics Spot to rotate $90^\circ$ in pitch from a four-legged stance to a balanced handstand on its front legs. The policy is a multilayer perceptron (MLP) mapping a proprioceptive observation to 12 joint position offsets. The reward shapes for the handstand pose and penalizes excessive joint torques and rapid action changes, with a large terminal penalty for falls. Episodes terminate on illegal contact between any body link and the ground. The full task specification, including specific observations, reward terms, termination conditions and domain randomization~\cite{tobinDomainRandomizationTransferring2017}, is in Appendix~\ref{apx-mdp}.

We compare two training conditions: \emph{with curriculum}, in which joint stiffness anneals from $K_0$ to $K_*$ according to Algorithm~\ref{alg:curriculum}, and \emph{without curriculum}, in which stiffness is fixed at the system-identified value $K_*$. Each condition is trained for 10 seeds; all other hyperparameters are identical.

\begin{figure}[t]
    \centering
    \includegraphics[width=0.95\linewidth]{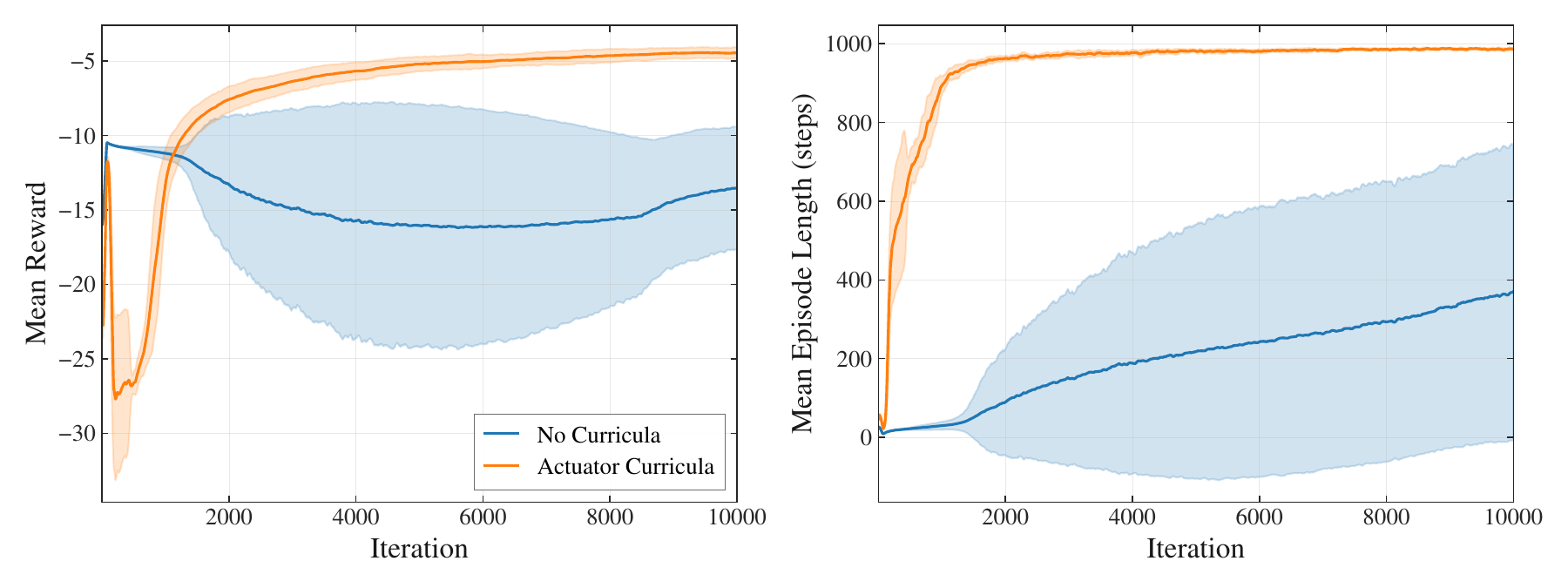}
    \caption{Training curves on the Spot handstand task with 10 seeds. Left: Mean reward. Right: Mean episode length. Without the curriculum, episode length remains around 400 steps after 10,000 policy updates.}
    \label{fig:training_curves}
\end{figure}

\begin{wrapfigure}{l}{0.45\textwidth}
    \centering
    \includegraphics[width=\linewidth]{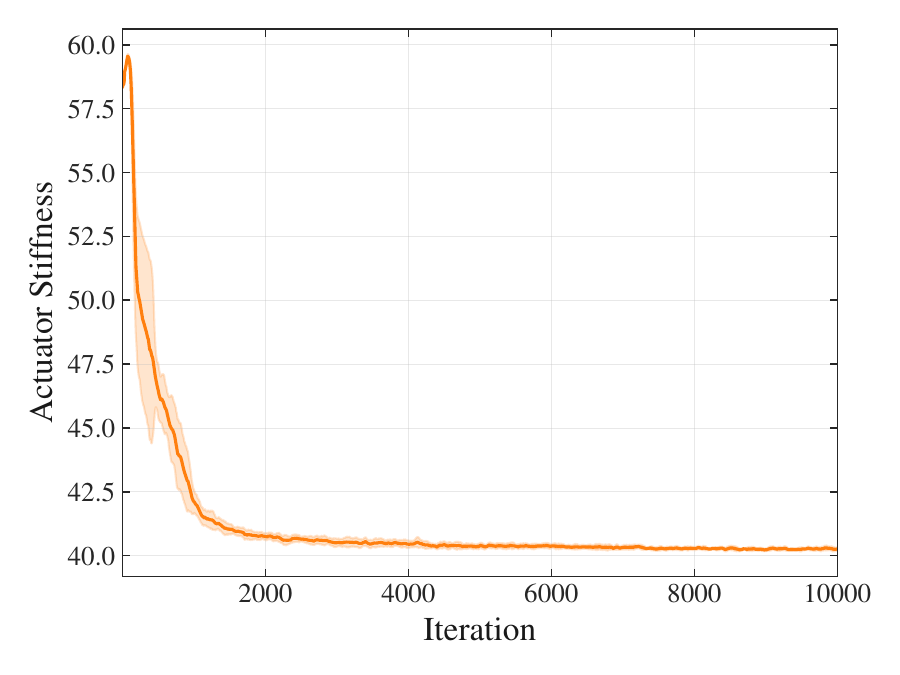}
    \caption{Stiffness $K(t)$ during training under the Actuator Dynamics Curriculum.}
    \label{fig:stiffness_trajectory}
\end{wrapfigure}

Figure~\ref{fig:training_curves} shows the resulting training curves. Without the curriculum, mean episode length grows slightly but plateaus around 400 steps and reward remains low; rollouts terminate early enough that the policy gradient does not produce a coherent training signal, consistent with the narrow-viability characterization of Section~\ref{sec:problem}. With the curriculum, both quantities rise monotonically and converge to their maximum values within $1.5$ hours of wall-clock training. Figure~\ref{fig:stiffness_trajectory} shows the corresponding stiffness trajectory $K(t)$: the schedule anneals from $K_0$ all the way to $K_*$ over the course of training, so the converged policy operates under the system-identified actuator parameters. Figure~\ref{fig:sim_stopmotion} compares the two policies qualitatively: the curriculum-trained policy executes a clean transition into the inverted equilibrium, while the baseline falls eventually. On hardware, we deployed all 10 curriculum-trained policies to the physical Spot as seen in Figure~\ref{fig:cover-real}. Each produced the handstand transition and held the inverted pose under physical disturbances applied by prodding the robot with a pole. On carpeted surfaces, the transition succeeded on the first attempt, while on padded surfaces it sometimes required a second attempt, which the policy executed as an emergent retry. Under significantly larger pushes the robot toppled onto its back and it could not recover from this configuration. In total, across approximately 20 recorded trials on surfaces including carpet, soft padding, and hardwood, all transitions succeeded. Representative deployments are included in the supplementary video.

\subsection{Ablation Study}
\label{sec:ablations}

We compare the Actuator Dynamics Curriculum against five ablations stated in Table~\ref{tab:ablation}, each varying one design choice while keeping all other hyperparameters identical. Each condition is trained for 10 seeds and evaluated over 1000 episodes under the system-identified deployment dynamics $K_*$. The final mean episode length and mean reward are reported in Table~\ref{tab:ablation}.

We additionally test whether the benefit of ADC can be attributed to effective action scaling or exploration-noise annealing. These ablations are reported in Appendix~\ref{sec:additional_ablations}.

\begin{table}[t]
    \centering
    \caption{Ablation study on the Spot handstand task. Mean $\pm$ std across 10 seeds, each evaluated over 1000 episodes.}
    \label{tab:ablation}
    \begin{tabular}{lcc}
        \toprule
        Condition & Episode Length & Reward \\
        \midrule
        Train at $K_*$ & $385 \pm 47$ & $-42.1 \pm 8.7$ \\
        Train at $K_0$ & $215 \pm 38$ & $-68.5 \pm 12.3$ \\
        Randomize $K \in [K_*, K_0]$ & $658 \pm 89$ & $-18.7 \pm 6.2$ \\
        Time-based annealing schedule & $751 \pm 62$ & $-12.3 \pm 5.1$ \\
        No early termination & $547 \pm 56$ & $-35.4 \pm 7.4$ \\
        \midrule
        \textbf{Ours} (Actuator Dynamics curriculum) & $\mathbf{975 \pm 12}$ & $\mathbf{-4.76 \pm 0.83}$ \\
        \bottomrule
    \end{tabular}
\end{table}

\section{Limitations}
\label{sec:limitations}

The analytical result in Section~\ref{sec:propositions} is established on a representative cart-pole system; lifting them to the full multi-rigid-body dynamics of a legged robot would close an interesting theoretical gap. The curriculum is currently demonstrated on the handstand transition task alone, but we are interested in characterizing its behavior on a broader class of narrow-viability tasks. Within the method itself, the schedule currently varies a single scalar stiffness uniformly across joints, with hyperparameters $K_0$, $L_\text{min}$, $L_\text{max}$, $\alpha$, and $p$ chosen by hand; per-joint schedules with joint-specific targets, and principled or adaptive hyperparameter selection, are possible extensions of this work. Finally, hardware deployment is currently reported qualitatively; a more rigorous evaluation with quantitative success metrics and an explicit failure-mode analysis would further validate the sim-to-real transfer.

\begin{figure}[t]
    \centering
    \begin{subfigure}{0.49\linewidth}
        \centering
        \includegraphics[width=\linewidth]{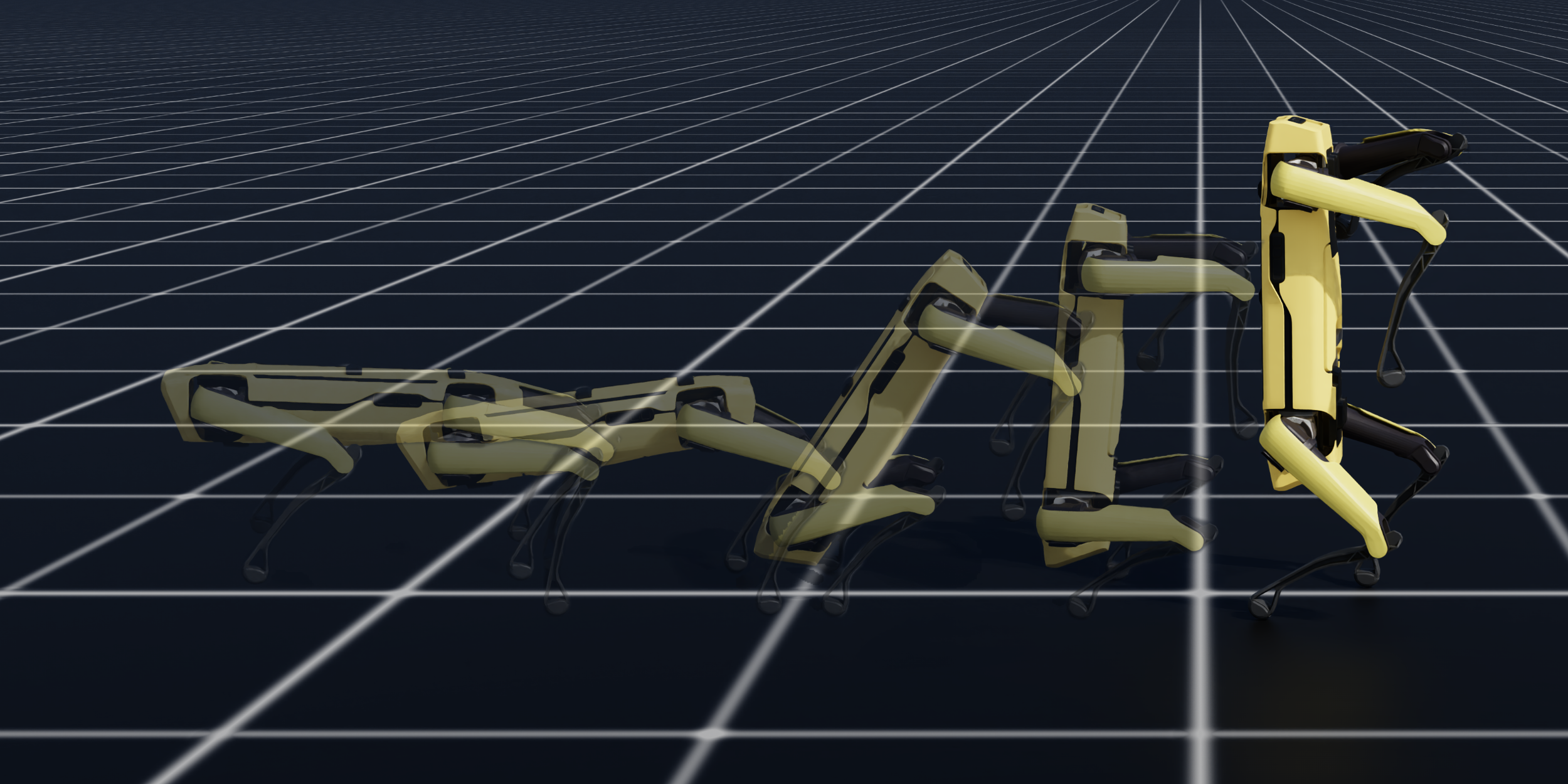}
        \caption{With curriculum}
        \label{fig:good}
    \end{subfigure}
    \hfill
    \begin{subfigure}{0.49\linewidth}
        \centering
        \includegraphics[width=\linewidth]{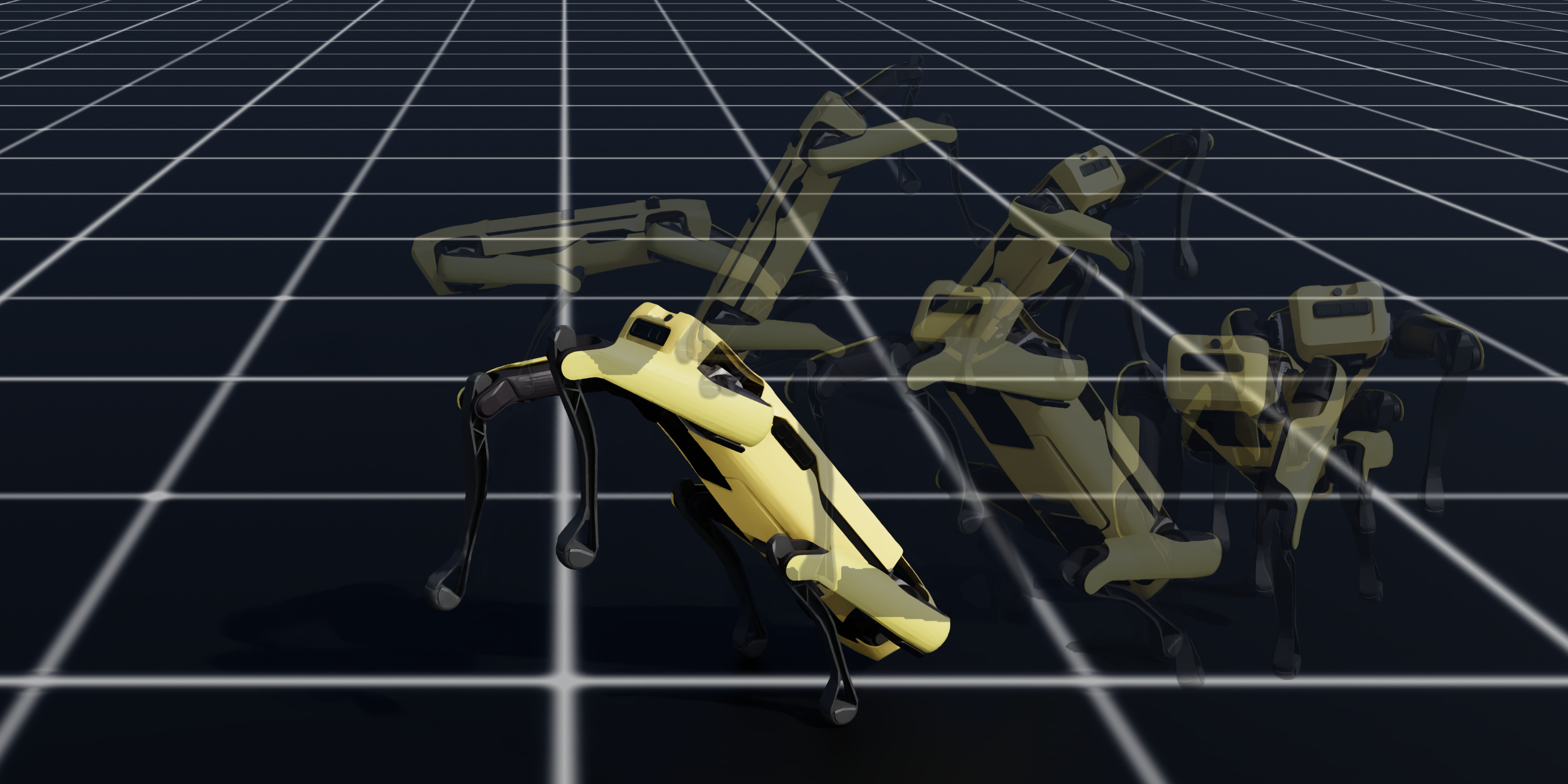}
        \caption{Without curriculum}
        \label{fig:bad}
    \end{subfigure}
    \caption{Snapshots of the handstand transition in simulation after $10000$ policy updates. The curriculum-trained policy \textbf{(a)} executes a clean transition into the inverted equilibrium whereas the baseline \textbf{(b)} fails.}
    \label{fig:sim_stopmotion}
\end{figure}

\section{Conclusion}
\label{sec:conclusion}
We observed that training a handstand transition policy on the Boston Dynamics Spot under system-identified actuator dynamics failed to converge, with most rollouts terminating before producing useful gradient signal. This work formalizes that observation through the notion of narrow-viability and addresses it with the \emph{Actuator Dynamics Curriculum}, a method that initializes joint stiffness at a high value and anneals it toward the system-identified value as completed episode lengths grow. The curriculum is motivated by a result on a representative cart-pole system where higher closed-loop joint natural frequency under critical damping enlarges the viability kernel and the fraction of initial states from which the task is feasible. Empirically, the curriculum converges on the handstand task and transfers zero-shot to hardware across different seeds. More broadly, our results suggest that simulated closed-loop joint dynamics is a useful axis along which to design curricula for tasks in which exploration is bottlenecked by termination conditions rather than by reward signal.

\clearpage

\bibliography{references}  %

\newpage

\section{Appendix}

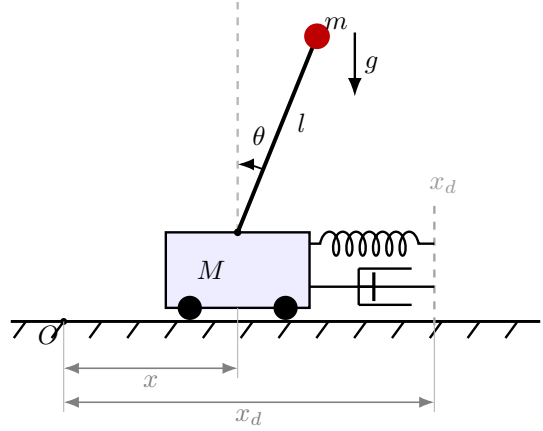
\begin{figure}[t]
\centering
\begin{tikzpicture}[
  >=Latex, line width=0.9pt,
  spring/.style={decorate,decoration={coil,aspect=0.5,segment length=5pt,amplitude=4.5pt,pre length=4pt,post length=4pt}},
  dim/.style={<->,gray,line width=0.6pt},
  ext/.style={gray!60,line width=0.4pt},
]
\def\xc{2.3}\def\xd{4.9}\def\cw{1.9}\def\ch{1.0}\def\pl{2.8}\def\ang{22}\def\yb{0.18}
\pgfmathsetmacro{\xR}{\xc+\cw/2}
\pgfmathsetmacro{\ytop}{\yb+\ch}
\pgfmathsetmacro{\yspring}{\yb+0.85*\ch}
\pgfmathsetmacro{\ydamp}{\yb+0.28*\ch}

\draw[line width=1.1pt] (-0.7,0) -- (6.3,0);
\foreach \i in {0,...,13}{\draw (-0.5+0.5*\i,0) -- ++(-0.16,-0.22);}

\draw[dashed,gray!70] (\xd,-0.05) -- (\xd,\ytop+0.35);
\node[gray!80] at (\xd+0.12,\ytop+0.62) {$x_d$};

\draw[spring] (\xR,\yspring) -- (\xd,\yspring);
\node at ($({(\xR+\xd)/2},\yspring)+(0,0.42)$) {};
\pgfmathsetmacro{\pist}{\xR+0.85}
\pgfmathsetmacro{\cylL}{\xd-1.0}
\draw (\xR,\ydamp) -- (\pist,\ydamp);
\draw[line width=1.2pt] (\pist,\ydamp-0.19) -- (\pist,\ydamp+0.19);
\draw (\xd,\ydamp) -- (\cylL,\ydamp);
\draw (\cylL,\ydamp-0.24) -- ($(\cylL,\ydamp-0.24)+(0.7,0)$)
      (\cylL,\ydamp+0.24) -- ($(\cylL,\ydamp+0.24)+(0.7,0)$)
      (\cylL,\ydamp-0.24) -- (\cylL,\ydamp+0.24);
\node at ($({(\xR+\xd)/2+0.25},\ydamp)+(0,-0.36)$) {};

\draw[fill=blue!7] (\xc-\cw/2,\yb) rectangle (\xc+\cw/2,\ytop);
\node at (\xc-0.35,\yb+\ch/2) {$M$};
\fill (\xc-\cw/3,\yb) circle (0.16);
\fill (\xc+\cw/3,\yb) circle (0.16);

\coordinate (piv) at (\xc,\ytop);
\coordinate (tip) at ($(piv)+({90-\ang}:\pl)$);
\coordinate (vtop) at ($(piv)+(0,\pl+0.25)$);
\draw[dashed,gray!55] (piv) -- (vtop);
\draw[line width=1.5pt] (piv) -- (tip);
\fill[red!75!black] (tip) circle (0.17);
\node[above right=-1pt] at (tip) {$m$};
\fill (piv) circle (0.05);
\pic[draw,line width=0.7pt,->,"$\theta$"{xshift=1pt},angle radius=9mm,angle eccentricity=1.45]
   {angle=tip--piv--vtop};
\node at ($(piv)!0.55!(tip)+(0.27,0.02)$) {$l$};

\draw[->] ($(tip)+(0.5,0.05)$) -- ++(0,-0.85) node[right,midway]{$g$};

\fill (0,0) circle (1.3pt);
\node[below left=-2pt] at (0,0) {$O$};
\draw[ext] (0,0) -- (0,-1.2);
\draw[ext] (\xc,\yb) -- (\xc,-0.75);
\draw[ext] (\xd,0) -- (\xd,-1.2);
\draw[dim] (0,-0.62) -- (\xc,-0.62) node[midway,below=-1pt]{$x$};
\draw[dim] (0,-1.07) -- (\xd,-1.07) node[midway,below=-1pt]{$x_d$};
\end{tikzpicture}
\caption{Cart-pole with PD control}
\label{fig:cartpole}
\end{figure}

\subsection{Proofs}
\label{apx-proofs}

\paragraph{Representative system.}
We instantiate $\mathcal{M}_{\omega_n}$ on a cart--pole with cart mass $M$, pole mass $m$, and pole length $l$ as shown in Figure~\ref{fig:cartpole}. The action is the commanded cart position $x_d$, drawn from a bounded set $|x_d| \le \bar{x}_d$. A PD controller on the cart with gains $(K,B)$ generates the cart force. We set $K = M\omega_n^2$ and $B = 2M\omega_n$, so that the closed-loop cart dynamics are critically damped ($\zeta = 1$) with natural frequency $\omega_n$. With state $s = (x,\dot x,\theta,\dot\theta)$, the closed-loop dynamics are
\begin{align}
\ddot{x} &= \frac{M\omega_n^2(x_d-x) - 2M\omega_n\dot{x}
   + ml\sin\theta\,\dot\theta^2 - mg\sin\theta\cos\theta}{M + m\sin^2\theta}, \\
\ddot{\theta} &= \frac{(M+m)g\sin\theta - \cos\theta\big(M\omega_n^2(x_d-x)
   - 2M\omega_n\dot{x} + ml\sin\theta\,\dot\theta^2\big)}{l(M + m\sin^2\theta)}.
\end{align}
The termination set is the complement of the admissible box $X = \{\,|x|\le\bar{x},\ |\theta|\le\bar\theta\}$, and $\mathcal{V}_T(\mathcal{M}_{\omega_{n}})$ is as in Definition~\ref{def:viab}.

\paragraph{Bounded-command feasibility assumptions.}
The proof of Proposition~\ref{prop:kernel} uses two assumptions on the bounded-command feasibility of the kernel:
\begin{enumerate}[label=(A\arabic*)]
    \item \textbf{Command margin.} Every $s_0 \in \mathcal{V}_T(\mathcal{M}_{\omega_{n,1}})$ admits a viable command $x_d^{(1)}(\cdot)$ with margin $\delta > 0$, i.e.\ $|x_d^{(1)}(t)| \le \bar{x}_d - \delta$ for $t \in [0,T]$.
    \item \textbf{Closeness.} The pair $(\omega_{n,1}, \omega_{n,2})$ satisfies $\Delta(\omega_{n,1}, \omega_{n,2}) \le \delta$, where $\Delta$ is the inversion-correction bound defined in the proof.
\end{enumerate}

\subsubsection{Proof of Proposition~\ref{prop:kernel}}
\label{app:proof_kernel}
Linearizing about the upright equilibrium $(\theta,\dot\theta)=(0,0)$ reduces the dynamics to
\begin{align}
\ddot{x} &= \omega_n^2(x_d - x) - 2\omega_n\dot{x} - \tfrac{mg}{M}\,\theta,
   \label{eq:lin-cart}\\
\ddot{\theta} &= \tfrac{g}{l}\,\theta - \tfrac{1}{l}\,\ddot{x}.
   \label{eq:lin-pole}
\end{align}
In particular, $\omega_n$ enters only the cart equation~\eqref{eq:lin-cart}, while the pole equation~\eqref{eq:lin-pole} depends on the actuator solely through $\ddot{x}$. Inverting~\eqref{eq:lin-cart} for the command that produces a given $\ddot{x}$,
\begin{equation}
x_d \;=\; x + \tfrac{2}{\omega_n}\dot{x} + \tfrac{1}{\omega_n^2}\ddot{x} + \tfrac{mg}{M\omega_n^2}\,\theta.
\label{eq:inv}
\end{equation}

\begin{proof}
Fix $\omega_{n,1}\le\omega_{n,2}$ and $s_0\in\mathcal{V}_T(\mathcal{M}_{\omega_{n,1}})$. By Definition~\ref{def:viab} and (A1), some $x_d^{(1)}(\cdot)$ keeps the trajectory $\bigl(x(\cdot),\dot{x}(\cdot),\theta(\cdot),\dot{\theta}(\cdot)\bigr)$ in $X$ for $t\in[0,T]$ with $|x_d^{(1)}(t)|\le\bar{x}_d-\delta$. Let $x_d^{(2)}(t)$ be~\eqref{eq:inv} evaluated at $\omega_{n,2}$ along the same trajectory. Substituting into~\eqref{eq:lin-cart} at $\omega_{n,2}$ recovers the same $\ddot{x}$, and~\eqref{eq:lin-pole} then yields the same $\ddot{\theta}$; with identical initial conditions, the trajectory is reproduced under $\omega_{n,2}$ and stays in $X$ on $[0,T]$.

For admissibility, subtracting~\eqref{eq:inv} at the two frequencies gives
\[
|x_d^{(2)}(t) - x_d^{(1)}(t)| \;\le\; \Delta(\omega_{n,1},\omega_{n,2}),
\]
where $\Delta\to 0$ as $\omega_{n,2}\to\omega_{n,1}$ and is bounded on viable trajectories. By (A2), $\Delta\le\delta$, hence $|x_d^{(2)}(t)|\le(\bar{x}_d-\delta)+\delta=\bar{x}_d$. Therefore $s_0\in\mathcal{V}_T(\mathcal{M}_{\omega_{n,2}})$, and since $s_0$ was arbitrary, $\mathcal{V}_T(\mathcal{M}_{\omega_{n,1}})\subseteq\mathcal{V}_T(\mathcal{M}_{\omega_{n,2}})$.
\end{proof}

\subsection{Task MDP Specification}
\label{apx-mdp}

\subsubsection{Observations}
\label{app:task:observations}

At each control step the policy receives a single concatenated observation vector $\mathbf{o}_t \in \mathbb{R}^{45}$, summarized in Table~\ref{tab:obs}. The base linear and angular velocities together with the projected gravity vector encode the base motion and orientation and joint positions are expressed relative to the robot's default standing pose.

\begin{table}[H]
\centering
\caption{Observation terms.}
\label{tab:obs}
\begin{tabular}{llc}
\toprule
Term & Symbol & Dim \\
\midrule
Base linear velocity & $\mathbf{v}_b$ & 3 \\
Base angular velocity & $\boldsymbol{\omega}_b$ & 3 \\
Projected gravity on base frame & $\hat{\mathbf{g}}$ & 3 \\
Joint positions & $\mathbf{q} - \mathbf{q}_0$ & 12 \\
Joint velocities & $\dot{\mathbf{q}}$ & 12 \\
Previous action & $\mathbf{a}_{t-1}$ & 12 \\
\midrule
Total & $\mathbf{o}_t$ & 45 \\
\bottomrule
\end{tabular}
\end{table}

\subsubsection{Actions}
\label{sec:task_actions}

The policy outputs 12 relative joint-position commands $a \in \mathcal{A}$. Actions are clipped component-wise to $a_i \in [-1,1]$ and converted to joint-position setpoints according to
\begin{equation}
    q_d = q_0 + \alpha a,
\end{equation}
where $q_0$ is the nominal standing configuration and $\alpha$ is the
action scale used throughout training.

\subsubsection{Rewards}
\label{app:task:rewards}

The reward function is a weighted sum of penalty terms. The policy is rewarded for maintaining the handstand pose at the target height and orientation, keeping the hind feet off the ground, and producing smooth low-effort actions. The terms and their weights are listed in Table~\ref{tab:rewards}.

\begin{table}[H]
\centering
\caption{Reward terms and weights.}
\label{tab:rewards}
\begin{tabular}{llr}
\toprule
Term & Description & Weight \\
\midrule
Base orientation & Deviation from target orientation (L2) & $-1.0$ \\
Base height & Deviation from target height (L2) & $-0.5$ \\
Hind-foot contact & Either hind foot contacting the ground & $-10.0$ \\
Base linear velocity & Linear velocity of the base (L2) & $-0.01$ \\
Base angular velocity & Angular velocity of the base (L2) & $-0.05$ \\
Joint torques & Joint torque magnitudes (L2) & $-5 \times 10^{-4}$ \\
Action rate & Change in commanded action between steps (L2) & $-0.01$ \\
Termination & Body or any upper leg link contacts the ground & $-500$ \\
\bottomrule
\end{tabular}
\end{table}

\subsubsection{Domain Randomization}
\label{app:task:randomization}

To improve transfer from simulation to hardware, several physical parameters are randomized during training. Randomizations are applied at three points: once at environment startup, at every episode reset, and at random intervals during episodes. The full list is given in Table~\ref{tab:dr}.

\begin{table}[H]
\centering
\caption{Domain randomization parameters. All ranges are sampled uniformly.}
\label{tab:dr}
\begin{tabular}{ll}
\toprule
Parameter & Range \\
\midrule
\multicolumn{2}{l}{\textit{Startup}} \\
Robot static friction & $[0.3, 1.0]$ \\
Robot dynamic friction & $[0.3, 0.8]$ \\
Base mass perturbation & $\pm 2.5$ kg \\
\multicolumn{2}{l}{\textit{Reset}} \\
Base position $(x, y)$ & $\pm 0.5$ m \\
Base yaw & $\pm \pi$ rad \\
Base linear velocity & $(\pm 1.5,\, \pm 1.0,\, \pm 0.5)$ m/s \\
Base angular velocity & $(\pm 0.7,\, \pm 0.7,\, \pm 1.0)$ rad/s \\
Joint position offset & $\pm 0.2$ rad \\
Joint velocity & $\pm 2.5$ rad/s \\
\multicolumn{2}{l}{\textit{Interval}} \\
Push interval & $[2,\, 5]$ s \\
Push velocity & $\pm 0.5$ m/s \\
\bottomrule
\end{tabular}
\end{table}

\subsection{Training and Simulation Parameters} 
\label{apx-train-sim-params}

\begin{table}[H]
\centering
\caption{IsaacSim Environment Configuration}
\begin{tabular}{ll}
\toprule
\multicolumn{2}{c}{\textbf{Simulation Parameters}} \\
\midrule
Physics timestep & 0.002 s (500 Hz) \\
Control decimation & 10 \\
Episode length & 20 s (1000 policy steps) \\
Parallel environments & 4096 \\
\bottomrule
\end{tabular}
\label{tab:sim_config}
\end{table}

\begin{table}[H]
\centering
\caption{PPO Configuration}
\begin{tabular}{ll}
\toprule
\multicolumn{2}{c}{\textbf{PPO Hyperparameters}} \\
\midrule
Rollouts & 24 \\
Learning epochs & 5 \\
Mini batches & 4 \\
Discount factor ($\gamma$) & 0.99 \\
GAE lambda ($\lambda$) & 0.95 \\
Learning rate & $1 \times 10^{-3}$ \\
Learning rate scheduler & Adaptive (KL-based) \\
KL threshold (scheduler) & 0.01 \\
Ratio clip & 0.2 \\
Value clip & 0.2 \\
Entropy loss scale & 0.0025 \\
Value loss scale & 0.5 \\
Gradient norm clip & 1.0 \\
Initial policy noise std & 1.0 \\
Total training iterations & 10{,}000 \\
\bottomrule
\end{tabular}
\label{tab:ppo_config}
\end{table}

\begin{table}[H]
\centering
\caption{Network Configuration}
\begin{tabular}{ll}
\toprule
\multicolumn{2}{c}{\textbf{Network Architecture}} \\
\midrule
\multicolumn{2}{l}{\textit{Policy network}} \\
Input & $\boldsymbol{o}_t$ \\
Observation normalization & disabled \\
Hidden layers & [512, 256, 128] \\
Activations & elu \\
Output & $\boldsymbol{a}_t$ \\
\multicolumn{2}{l}{\textit{Value network}} \\
Input & $\boldsymbol{o}_t$ \\
Observation normalization & disabled \\
Hidden layers & [512, 256, 128] \\
Activations & elu \\
Output & $V(\boldsymbol{o}_t)$ \quad (value estimate) \\
\bottomrule
\end{tabular}
\label{tab:network_config}
\end{table}

\begin{table}[H]
    \centering
    \caption{Actuator Dynamics Curriculum hyperparameters.}
    \label{tab:adc_hyperparams}
    \begin{tabular}{lc}
        \toprule
        Parameter & Value \\
        \midrule
        Initial stiffness $K_0$ & $60$ \\
        Target stiffness $K^*$ & $40$ \\
        Episode-length threshold $L_{\min}$ & 0 \\
        Episode-length threshold $L_{\max}$ & 1000 \\
        EMA smoothing factor $\eta$ & 0.05 \\
        Annealing exponent $p$ & $1$ \\
        \bottomrule
    \end{tabular}
\end{table}

\subsection{Additional Ablations}
\label{sec:additional_ablations}

We evaluate whether ADC's benefit can be explained by action scaling or exploration noise using three additional ablations: retuned action scaling at $K^*$, action-normalized ADC with $K(t)\alpha(t)=K^*\alpha^*$, and an episode-length-driven action-noise curriculum. Training curves are shown in Fig.~\ref{fig:training-curves} and final results in Table~\ref{tab:ablation-new}. The normalized-action ADC performs close to the original ADC, while the other two conditions provide only modest gains over the baseline.

\begin{figure}[b]
  \centering
  \includegraphics[width=\linewidth]{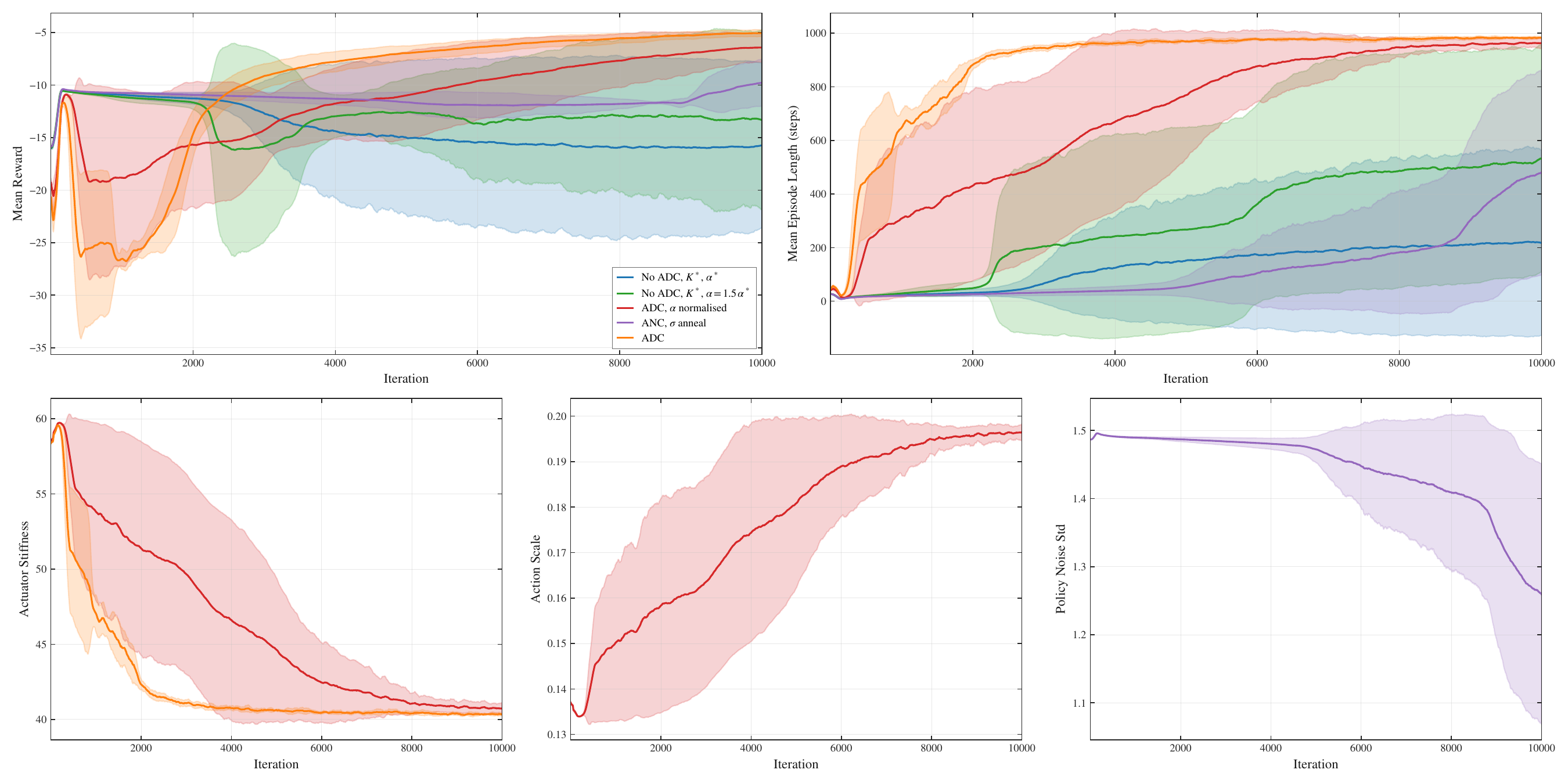}
  \caption{Training curves for the additional controlled ablations.}  
  \label{fig:training-curves}
\end{figure}

\begin{table}[t]
    \centering
    \setlength{\tabcolsep}{4pt}
    \caption{Additional controlled ablations on action scaling and exploration noise. Mean $\pm$ standard deviation across 10 seeds, with each policy evaluated over 1000 episodes under the system-identified dynamics $K^*$.}    \label{tab:ablation-new}
    \begin{tabular}{lcc}
        \toprule
        Condition & Ep. Length & Reward \\
        \midrule
        Train at $K_*$ (baseline)        & $385 \pm 47$ & $-42.1 \pm 8.7$ \\
        Train at $K_*$, retuned $\alpha$ & $453 \pm 52$ & $-37.4 \pm 7.4$ \\
        Train at $K_*$, noise curriculum & $447 \pm 48$ & $-32.1 \pm 6.3$ \\
        Ours, normalised $\alpha$        & $952 \pm 40$ & $-5.4 \pm 1.5$ \\
        \midrule
        \textbf{Ours} (ADC)              & $\mathbf{975 \pm 12}$ & $\mathbf{-4.76 \pm 0.83}$ \\
        \bottomrule
    \end{tabular}
  \vspace{-18pt}    
\end{table}

\end{document}